\documentclass[nohyperref]{article}

\usepackage{microtype}
\usepackage{graphicx}
\usepackage{subfigure}
\usepackage{booktabs} 

\usepackage{hyperref}

\newcommand{\pd}[2]{\frac{\partial#1}{\partial#2}}
\newcommand{\ppm}{$\pm$}

\usepackage[arxiv]{icml2023}

\usepackage{amsmath}
\usepackage{amssymb}
\usepackage{mathtools}
\usepackage{amsthm}
\usepackage{bm}

\usepackage[capitalize,noabbrev]{cleveref}

\usepackage{cases}

\theoremstyle{plain}
\newtheorem{theorem}{Theorem}[section]
\newtheorem{proposition}[theorem]{Proposition}

\theoremstyle{definition}

\theoremstyle{remark}
\newtheorem{remark}[theorem]{Remark}

\usepackage[textsize=tiny]{todonotes}
\usepackage{placeins}

\icmltitlerunning{Harmonic (Quantum) Neural Networks}

\begin{document}

\twocolumn[
\icmltitle{Harmonic (Quantum) Neural Networks}



\icmlsetsymbol{equal}{*}

\begin{icmlauthorlist}
\icmlauthor{Atiyo Ghosh}{pasqal}
\icmlauthor{Antonio A. Gentile}{pasqal}
\icmlauthor{Mario Dagrada}{pasqal}
\icmlauthor{Chul Lee}{lg}
\icmlauthor{Seong-Hyok Kim}{lg}
\icmlauthor{Hyukgeun Cha}{lg}
\icmlauthor{Yunjun Choi}{lg}
\icmlauthor{Dongho Kim}{lg,posco}
\icmlauthor{Jeong-Il Kye}{lg,pino}
\icmlauthor{Vincent Elfving}{pasqal}
\end{icmlauthorlist}

\icmlaffiliation{pasqal}{PASQAL SAS, 2 av. Augustin Fresnel, Palaiseau, 91220, France}
\icmlaffiliation{lg}{LG Electronics, AI Lab, CTO Div, 19, Yangjae-daero 11-gil, Seocho-gu, Seoul, 06772, Republic of Korea}
\icmlaffiliation{posco}{POSCO Holdings, AI R\&D Laboratories, 440, Tehera-ro, Gangam-gu, Seoul, 06194, Republic of Korea}
\icmlaffiliation{pino}{Pinostory Inc., 1-905 IT Castle, Seoul, 08506, Republic of Korea}
\icmlcorrespondingauthor{Atiyo Ghosh}{atiyo.ghosh@pasqal.com}

\icmlkeywords{Harmonic Functions, Physics Informed Machine Learning, Quantum Machine Learning}

\vskip 0.3in
]



\printAffiliationsAndNotice{}  

\begin{abstract}
Harmonic functions are abundant in nature, appearing in limiting cases of Maxwell's, Navier-Stokes equations, the heat and the wave equation.
Consequently, there are many applications of harmonic functions from industrial process optimisation to robotic path planning and the calculation of first exit times of random walks.
Despite their ubiquity and relevance, there have been few attempts to incorporate inductive biases towards harmonic functions in machine learning contexts.
In this work, we demonstrate effective means of representing harmonic functions in neural networks and extend such results also to quantum neural networks to demonstrate the generality of our approach.
We benchmark our approaches against (quantum) physics-informed neural networks, where we show favourable performance.
\end{abstract}

\section{Introduction}

Harmonic functions can be defined as the solutions to Laplace's equation $\nabla^2\phi = 0$, though there are alternative mathematically-equivalent definitions.
Such functions are ubiquitous in nature, appearing in limiting cases of Maxwell's equations~\cite{griffiths2005introduction}, the heat (or diffusion) equation, the wave equation, irrotational flow in fluid dynamics~\cite{anderson2011ebook} and first-hitting times of random walks~\cite{redner2001guide}.
Consequently, harmonic functions are relevant in many practical settings, including navigation in robots~\cite{edson2002}, electrostatic imaging~\cite{akduman2002} and heat transport~\cite{sharma2018}, to name but a few examples.
Finite element methods (FEM) often represent the state-of-art approach in solving the corresponding equations, with attempts specifically targeting harmonic functions \cite{galybin2010}. In contrast to FEM, machine learning based approaches offer the chance to assist with data-driven modelling 
(analagous to \cite{raissi2019}), handle noisy-data/boundary conditions, solve inverse problems for design optimization \cite{lu2021}, perform transfer learning and model discovery \cite{both2021deepmod}. Effective learnable and differentiable representations of harmonic functions could therefore be key in multiple fields. 

\paragraph{Physical Inductive Biases in Conventional Machine Learning}
Of particular relevance to our work on harmonic functions are physics-informed neural networks (PINNs)~\cite{raissi2019}, which use derivatives of neural networks (NNs) with respect to their inputs as a regularisation term to provide an inductive bias towards a given differential equation, including potentially towards harmonic functions.
Even though this concept was introduced decades ago~\cite{lagaris1998}, PINNs have only recently garnered much attention.
For a wider overview of PINNs, we direct the interested reader to a recent review~\cite{karniadakis2021}.

Related to PINNs, there has been a growing recent interest in applying machine learning towards differential equation problems, such as the solution of families of differential equations~\cite{li2020, lu2019}, or facilitating the solution via dimensionality reduction~\cite{gunzburger2007}.
Finding NNs which directly obey a given (linear) differential operator as a hard constraint has attracted some recent attention~\cite{hendriks2020}. Divergence-free vector fields were also used recently to represent continuity equations~\cite{richter2022neural}. We also leverage divergence-free fields in parts of our exposition.
In addition to physical inductive biases encodable via differentiable equations, there have also been recently developed methods to effectively impose energy conservation on learnt representations~\cite{greydanus2019, cranmer2020}.

\paragraph{Physical Inductive Biases in Quantum Machine Learning}
A promising avenue for achieving practically-relevant quantum advantage within currently available and near-term noisy intermediate-scale quantum computers consists of variational quantum algorithms~\cite{cerezo2021}. We introduce this perspective in further detail in Appendix \ref{app:qalgorithms} and refer interested readers to a recent introduction to quantum machine learning~\cite{schuld2021machine} for further background.
There are several parallels between variational quantum models and NNs.  
For example, both can be seen as parameterised functions which are optimised with a conventional optimiser, achieving universal approximation properties~\cite{Hornik1989, Goto2020}. 
Furthermore, there are means of achieving automatic differentiation directly on quantum circuits~\cite{guerreschi2017, kyriienko2021}, making them effectively a possible generalization of classical NN architectures. 
It has been speculated that including inductive biases within quantum circuit designs might prove beneficial~\cite{hadfield2019, bharti2021, kubler2021}, possibly also to address difficulties that hinder the achievement of practical quantum advantage \cite{Huang2021}. 

Differentiable Quantum Circuits (DQC), outlined in~\cite{kyriienko2021dqc}, adopt an advantageous latent space representation for the mapping of input data, and derivatives of quantum circuits with respect to their inputs to solve given (also non-linear) differential equations.
Equipped with appropriate loss functions taking into account soft constraints, as well as automatic differentiation rules providing an effective surrogate of backpropagation rules in conventional NNs, such architectures can then be regarded as a quantum PINN, or for simplicity qPINNs (we emphasise though that this terminology is \emph{not} widely used).
Training qPINNs poses in principle no fewer issues than PINNs of equivalent expressivity, and this limitation might be exacerbated by the sheer number of parameters that can potentially describe vast quantum circuits, which can complicate training \cite{mcclean2018, anschuetz2022beyond}. 
This makes it an ideal test-bed to investigate the application of relevant inductive biases in quantum circuits.
However, the use of inductive biases in designing quantum circuits and quantum machine learning models is still a nascent field, that we review briefly in general in Appendix \ref{app:qalgorithms}.

\paragraph{Focus of the Paper}
Our main focus in this work is to incorporate inductive biases for harmonic functions in machine learning models.
We demonstrate exactly-harmonic functions for simply-connected two-dimensional domains using complex-valued NNs. For multiply-connected domains, we provide a domain decomposition methodology which allows us to continue using complex-valued, conventional NNs as harmonic function representations.
We investigate the general applicability of the methods we develop, by demonstrating how they can be generalised to the quantum architectures introduced above. 
Finally, we use example problem settings from heat distribution, electrostatics and robot navigation to demonstrate the effectiveness of our proposed inductive biases in both conventional and (simulated) quantum NNs.
We provide implementations of our classical architectures in the supplementary material \cite{SI} to facilitate the adoption of these inductive biases.

\paragraph{Note on Terminology} 
To avoid confusion, henceforth we will use the term `neural network' to refer to statements applying to both quantum circuits and conventional neural networks.
When seeking to refer specifically to quantum circuits or conventional NNs, we will explicitly state so.

\section{Theory}

We start by considering exact harmonicity in two dimensions in section \ref{sec:holomorphic} and \ref{sec:multiholomorphic} using techniques leveraging the behaviour of complex functions.
However, such constructions do not extend to arbitrary dimensions. Consequently, we explore higher dimensions in section \ref{theory:curl}.

\subsection{Exact Harmonicity in Two Dimensions}
Consider (quantum and classical) neural networks of $\phi_H: \mathbb{R}^2 \rightarrow \mathbb{R}$ of the form

\begin{equation}\label{eq:harmonic_simply_connected}
\phi_{H}(x,y) = \Re\circ \textrm{NN}_h \circ c.
\end{equation}

where $c(x, y) = x + iy$ and $\Re(x + iy) = x$.

\begin{theorem}
\label{thm:simply_connected_holomorphic}
Architectures of the form in equation \eqref{eq:harmonic_simply_connected} are harmonic for holomorphic (i.e. complex differentiable) choices of $\textrm{NN}_h$.
\end{theorem}
\begin{proof}
Since $\textrm{NN}_h$ is holomorphic, it satisfies the Cauchy-Riemann equations. Writing $\textrm{NN}_h(x + iy) = u(x,y) + v(x, y)$, the Cauchy-Riemann equations then read
\begin{equation} \label{eq:cauchy_riemann}
\frac{\partial u}{\partial x} = \frac{\partial v}{\partial y} \quad\mathrm{and}\quad \frac{\partial u}{\partial y} = -\frac{\partial v}{\partial x} 
\end{equation}
Taking partial derivatives of the first equation of \eqref{eq:cauchy_riemann} with respect to $x$ and the second equation with respect to $y$ and summing both equations yields $\frac{\partial^2 u}{\partial x^2} + \frac{\partial^2 u}{\partial y^2} = 0$. 
\end{proof}

\begin{theorem}
There exist polynomial choices of $\textrm{NN}_h$ allow equation \eqref{eq:harmonic_simply_connected} to converge uniformly to any harmonic function on the interior of any simply-connected compact proper subdomain of $\mathbb{C}$.
\end{theorem}

\begin{proof}
Any simply-connected proper (open) subdomain $\mathbb{C}$ is homeomorphic to the (open) unit disc by the Riemann mapping theorem \cite{bak2010}. 
This demonstrates the complement of any simply-connected domain is connected in $\mathbb{C}$. Thus, Mergelyan's theorem applies, which guarantees that any continuous function which is holomorphic on the interior of the domain can be approximated uniformly by polynomials \cite{gaier1987lectures}.
\end{proof}

\begin{remark}
Even though (learnable) polynomials can provide a basis for universality arguments, representing arbitrary degree polynomials might be memory intensive for practical applications.
Other holomorphic constructions for $\textrm{NN}_h$ might thus provide better practical performance.
\end{remark}

However, architectures of the form in \eqref{eq:harmonic_simply_connected} are not universal for holomorphic $NN_h$ in the space of arbitrary compact domains in $\mathbb{C}$.

\begin{proposition}
\label{thm:not_universal_simply_connected}
Architectures of the form in equation \eqref{eq:harmonic_simply_connected} cannot represent arbitrary harmonic functions on multiply-connected domains for holomorphic $NN_h$.
\end{proposition}

\begin{proof}
As a counterexample, the function $u(x,y) = \log(x^2 + y^2)$ is harmonic on $\mathbb{C}\backslash{0}$, but is not the real part of any holomorphic function there \cite{bak2010}. As $NN_h$ equation \eqref{eq:harmonic_simply_connected} is restricted to be holomorphic, it cannot represent this function.
\end{proof}

Consequently, in this section, we provide separate techniques for modelling harmonic functions in simply-connected (section \ref{sec:holomorphic}) domains and multiply-connected domains (section \ref{sec:multiholomorphic}).

\subsubsection{Exact Harmonicity: Holomorphic Networks on Simply-Connected Domains}\label{sec:holomorphic}
\paragraph{Conventional Neural Networks}

Construct a complex-valued neural network $\textrm{NN}_h: \mathbb{C} \rightarrow \mathbb{C}$ by taking a multi-layer perceptron with complex-valued weights in linear layers, and holomorphic activation functions, for example $\sigma(z) = \sin(z)$ or $\sigma(z) = \exp(z)$.

Using such a network in equation \eqref{eq:harmonic_simply_connected} is then guaranteed to be harmonic by Theorem~\ref{thm:simply_connected_holomorphic}.

There has been much research into complex-valued conventional NNs, but we consider a meaningful review of them to be outside the scope of our work, referring to~\cite{bassey2021} for an introduction to their theory and application.

\paragraph{Quantum Neural Networks}
We define a quantum variational state $|\Psi_{\theta,x,y}\rangle$, parametrised by $\theta$-dependent unitaries $\hat{\mathcal{U}}_{\theta_k}$ as well as a quantum feature map (QFM) - e.g. of Chebyshev tower type \cite{kyriienko2021dqc} - used to port the input variables to the Hilbert space acted upon by $\hat{\mathcal{U}}_{\theta_k}$: 
\begin{equation}
|\Psi_{\theta,x,y}\rangle = \hat{\mathcal{U}}_{\theta_2}\textrm{QFM}(x,y)\hat{\mathcal{U}}_{\theta_1}|\emptyset\rangle; 
\end{equation} 
where $|\emptyset\rangle$ represents a reference state. 
A QNN is typically modelled as a measurement of a cost Hamiltonian $\hat{\mathcal{H}}$:
\begin{equation}
\phi_{Q}(x,y) = \langle \Psi_{\theta,x,y}|\hat{\mathcal{H}}|\Psi_{\theta,x,y}\rangle. 
\label{eq:simply_connected_qnn} 
\end{equation}
Here, instead, we modify a `quantum kernel' approach \cite{schuld2019quantum, Schuld2021} to introduce a harmonic architecture. 
\begin{theorem}
Adopting an imaginary feature map of the form $\textrm{IQFM}(x,y)=e^{-(x + i y)\hat{\mathcal{H}}}$, the output obtained from the circuit as: 
\begin{equation}
\phi_{QH}(x,y) = \Re\{\langle \emptyset|\Psi_{\theta,x,y}\rangle\} 
\label{eq:harmonic_simply_connected_qnn}
\end{equation}
is harmonic over the inputs $x, y$. 
\label{th:Qharmonicity}
\end{theorem}

\begin{remark}
 For $x=0$, IQFM$(x,y)$ represents a unitary evolution generated by $\hat{\mathcal{H}}$, similar to the Hamiltonian evolution feature maps in \cite{kyriienko2021dqc, kyriienko2021}. 
For non-zero $x$, the operation is generally non-unitary, requiring Quantum Imaginary Time Evolution (QITE) to be implemented with quantum gates in hardware \cite{McArdle2019,Yuan2019, Motta2020}.
\end{remark}
\begin{proof}
The bulk of the proof is reported for brevity in Appendix \ref{app:qholomorphic}, where we show how extending QITE formalism, we can symbolically represent: 
\begin{equation}\label{eq:harmonic_simply_connected_qnn_result}
\phi_{QH}(x,y) = \Re\{\sum_m c_m(\theta_1)  \langle \emptyset|\hat{\mathcal{U}}_{\theta_2} |\psi_m\rangle e^{-(x + i y)E_m}\}.
\end{equation}
where $E_m$ is the $m$-th (scalar) eigenvalue of Hamiltonian $\hat{\mathcal{H}}$ and the state $|\Psi\rangle$ was decomposed into the basis set $|\psi_m\rangle$ with coefficients $c_m(\theta_1)$, summing over all $2^N$ eigenvalues $m$ for an N-qubit Hamiltonian.
Proven Eq. \ref{eq:harmonic_simply_connected_qnn_result}, we can then observe how $\phi_{{QH}}$ represents the real part of a sum over an exponential number of holomorphic functions (with coefficients tuned by circuit parameters $\theta_1$ and $\theta_2$), and is therefore in total harmonic.
\end{proof}

We can thus generalise to harmonic quantum neural networks the construct in equation \eqref{eq:harmonic_simply_connected}, where $\textrm{NN}_h(x + iy) = \langle \emptyset|\Psi_{\theta,x,y}\rangle$.

\subsubsection{Exact Harmonicity: Multiholomorphic Networks on Multiply-Connected Domains}\label{sec:multiholomorphic}

For many problems of interest, harmonic functions on multiply-connected domains need to be constructed. As demonstrated in proposition~\ref{thm:not_universal_simply_connected}, we require architectures beyond those demonstrated in equation \eqref{eq:harmonic_simply_connected} in these scenarios. We develop such architectures in this section.

We consider a multiply-connected domain in $\mathbb{R}^2$ which we denote by $\Omega$ and disjointly decompose it such that $\Omega = (\dot{\bigcup}_i \omega_i) \dot{\cup} (\dot{\bigcup}_{i, j} \partial\omega_{i, j})$ with $i, j \in \mathbb{N}$, each $\omega_i$ being simply-connected and each $\partial\omega_{ij}$ a region of (Lesbegue) measure zero between the subdomains $\omega_i$ and $\omega_j$. 

For each $\omega_i$, we associate a (quantum) harmonic neural network, $\phi_{H}$, of the form outlined in Eq. \ref{eq:harmonic_simply_connected}.
We can then construct a representation for a harmonic neural network $\phi_{\textrm{MH}}: \Omega \rightarrow \mathbb{R}$ as follows:
\begin{equation}\label{eq:harmonic_multiply_connected}
\phi_{\textrm{MH}}(\bm{x}) = \begin{cases} 
      \phi_H^{(i)}(\bm{x}) & \text{$\bm{x} \in \omega_i$}\\
      \frac{\phi_H^{(i)}(\bm{x}) + \phi_H^{(j)}(\bm{x})}{2} & \text{$\bm{x} \in \partial\omega_{ij}$}.
\end{cases}
\end{equation}

\begin{proposition}
Architectures of the form in \eqref{eq:harmonic_multiply_connected} are exactly harmonic almost everywhere in $\Omega$. Furthermore, they uniformly converge to any harmonic function almost anywhere if each $\phi_H^{(i)}$ converges uniformly to any harmonic function on each $\omega_i$.
\end{proposition}

\begin{proof}
The only regions where $\phi_{MH}(x)$ might not be harmonic are in $\partial\omega_{ij}$, which have Lebesgue measure zero, hence allowing for harmonicity almost everywhere. Similarly, the universality of each $\phi_H^{(i)}$ on each $\omega_i$ means the universality of $\phi_{MH}$ can only be violated on each $\partial\omega_{ij}$, allowing for uniform convergence almost everywhere.
\end{proof}

While $\phi_{\textrm{MH}}$ is harmonic by construction on each $\omega_i$, it remains to make it harmonic across each $\partial\omega_{ij}$.
We achieve this by including an extra loss term $L_{\partial\omega}$ alongside any other loss that $\phi_{\textrm{MH}}$ is being optimised upon: \begin{equation}
\begin{split}
\label{eq:multiply_connected_loss}
    L_{\partial\omega} = & \sum_{i,j} \mathbb{E}_{\bm{x}\sim\partial\omega_{ij}}\left[\left(\phi_H^{(i)}(\bm{x}) - \phi_H^{(j)}(\bf x)\right)^2 + \right. \\
    & \left. \left\| \nabla \phi_H^{(i)}(\bm x) - \nabla \phi_H^{(j)}(\bm x) \right\|_2^2 \right],
\end{split}
\end{equation}

where the first term in the sum incentivises continuity in the harmonic function, and the second term represents a squared $L_2$ norm which incentivises continuity in the conservative field of the harmonic function.
Such stitching together of domains, with loss functions ensuring continuity of relevant dynamics, is closely related to some domain decomposition techniques in PINNs~\cite{jagtap2020}.

The fundamental result of such a decomposition is that $\phi_{\textrm{MH}}$ can be harmonic almost surely, but it need not correspond to any holomorphic function on each $\partial\omega_{ij}$.
Since this construction is only non-harmonic on boundaries, we consider it to be harmonic almost everywhere.
We consider the practical performance of this in Sect.\ref{sec:application_heat}.

\subsection{Approximate Harmonicity: Curl-Driven Harmonic Networks}\label{theory:curl}

While the previous sections demonstrate exact harmonicity in two dimensions, they are not applicable in higher dimensions. In this section, we present techniques applicable to arbitrary dimensions. Note that Laplace's equation can be written as $\nabla \cdot (\nabla \phi) = 0$. 
This demonstrates that the gradient of any harmonic function must be divergence-free. 
Furthermore, in three dimensions, a general divergence-free field can be obtained by the identity $\nabla \cdot (\nabla \times A) = 0$.
This motivates the construction of an inductive bias comprising of two networks: $\phi_C(\bm{x}; \bm\theta_\phi)$ representing a harmonic potential field and $A(\bm{x}, \bm\theta_A)$ representing a neural network whose curl will be taken to provide a representation of the underlying conservative field of $\phi_C$, where $\bm\theta_\phi$ and $\bm\theta_A$ represent trainable parameters.
Including the following loss term during network optimisation thus provides an inductive bias towards harmonic functions:
\begin{equation}\label{eq:curl_net_loss}
    L_c(\bm\theta_\phi, \bm\theta_A) = \mathbb{E}\left[\left\| \nabla\phi_C  - \nabla \times A \right\|_2^2 \right]
\end{equation}
where the expectation is taken with respect to a probability distribution whose support covers the domain over which the inductive bias is desired. During optimisation, both $\phi_C$ and $A$ can be trained to minimise \eqref{eq:curl_net_loss}, there is no minimax game as in generative adversarial networks that might destabilise training in such a procedure.

Even though the curl operator is only well-defined in three dimensions, we note that we can use the exterior calculus of differential forms to derive similarly divergence-free operators in arbitrary dimensions. 

\begin{proposition}
Given an $(N-2)$ form in $\mathbb{R}^N$ whose components in a Cartesian basis are given by a neural network $A: \mathbb{R}^d\rightarrow \mathbb{R}^N$, $d=\binom{N}{N-2}$, the exterior derivative $dA$ represents a divergence-free neural network.
\end{proposition}

\begin{proof}
The exterior derivative of an $(N-2)$ form $A$, yields an exact $(N-1)$ form, $dA$. Since there is a correspondence between the divergence of a field and the exterior derivative of an $(N-1)$ form, and since $d(d\phi) = 0$ for any differential form $\phi$, the result follows.
\end{proof}

While a full discussion of differential forms is outside the scope of this work, we provide example calculations in Appendix \ref{app:divfree} to facilitate incorporating such architectures in wider work.
We also note that similar arguments have been recently presented in \cite{richter2022neural}.

In two dimensions, this leads (with some abuse of notation in reusing the curl symbol to denote exterior derivatives) to:
\begin{equation}\label{eq:curl2d}
\nabla \times A = \left( \pd{y_1}{x_2}, -\pd{y_1}{x_1} \right),
\end{equation}
with $A: \mathbb{R}^2 \rightarrow \mathbb{R}$ defined by $A(x_1, x_2) = y_1$. A sample four-dimensional calculation is demonstrated in Appendix \ref{app:divfree_4d}.

We note that such operators are amenable to implementation on quantum circuits. 
Two independent quantum circuits representing $\phi_C$ and $A$ appearing in Eq,~\ref{eq:curl_net_loss} (as well as determining additional loss contributions detailed in Sect.~\ref{sec:methods}) can be estimated via simple cost functions (e.g. the \textit{total magnetization}).  
Composing such contributions into the relevant loss term(s) can be attained with additional classical automatic differentiation routines, e.g. \cite{bergholm2018pennylane}.
The latter ones can also handle circuit derivatives (here necessary up to the 2\textsubscript{nd} order), which can be alternatively attained with analytical methods on quantum hardware, such as the parameter-shift rule~\cite{guerreschi2017, mari2021estimating}.

\section{Applications}

We exemplify our approaches on sample applications geared towards electrostatics (Sect.\ref{sec:application_electrostatic}), heat distribution (Sects.\ref{sec:application_heat2}\&\ref{sec:application_heat}), and robot navigation (Sect.\ref{sec:application_robot}) and a further 3D fluid flow example in  Sect.\ref{app:results_3Dpipe}.
In order to provide robust benchmarks, we benchmark several forward solutions against finite element methods (FEM). Since the variational form of the FEM solution is linear, we assume that it converges reliably to provide a strong baseline.
In addition to the methods introduced in this work, we also include comparisons to PINNs~\cite{raissi2019}, as well as PINNs with Dirichlet-condition constrained architectures~\cite{lu2021}. 
We demonstrate a proof of concept of a quantum holomorphic network in Sect.\ref{sec:application_heat2}.
A multiply-connected domain is explicitly constructed in Sect.\ref{sec:application_heat} to demonstrate the Multiholomorphic networks of Sect.\ref{sec:multiholomorphic}. 
Multiholomorphic networks are not tested in other settings since their spatial domains do not require a simply-connected decomposition.
No holomorphic networks are tested in Sect.\ref{app:results_3Dpipe} since holomorphic networks do not apply to the 3D domain of that benchmark.

While we present results for both classical and quantum circuits, we emphasise that we do not wish to compare classical and quantum neural networks. The comparison is fraught with difficulties. For example, there is no clear way to compare a number of qubits with multilayer perceptron size.

Full finite elements and classical NN code  can be found in the Supplementary Material \cite{SI}.
We continue by outlining some loss functions we use with different methods, as well as summarising the implementation of each method.

\subsection{Methods}
\label{sec:methods}

\subsubsection{Loss Functions}

\paragraph{Dirichlet Losses} Imposing Dirichlet boundary conditions can be done variationally, using mean squared error losses.
Given a Dirichlet boundary $\Gamma$, where the target function is known to take constant value $c \in \mathbb{R}$, we train an arbitrary (quantum or conventional neural network) $\phi(\bm x, \bm \theta)$ to obey such a condition by including a term:
\begin{equation}\label{eq:dirichlet_loss}
    L_d = \mathbb{E}_{\bm x\sim\Gamma} \left[\left( \phi(\bm x, \bm \theta) - c \right)^2 \right]
\end{equation}

\paragraph{Physics-Informed Harmonic Losses} 
Given a spatial-domain $\Omega$ and a (quantum) neural network $\phi(\bm x, \bm{\theta})$, we follow established methodologies~\cite{raissi2019} to define a physics-informed loss towards harmonic functions by zeroing the Laplacian
\begin{equation}\label{eq:laplacian_loss}
    L_{\textrm{PI}_H} = E_{\bm x\sim\Omega} \left[\left(\nabla^2 \phi(\bm x, \bm \theta)\right)^2 \right]
\end{equation}

\paragraph{Dielectric Interface Loss}
To model the behaviour of a potential field across a dielectric interface, we must impose Maxwell's interface conditions, namely that potential fields (i.e. harmonic functions) are continuous across the interface, and that the normal components of underlying electric fields (i.e. the gradient of the harmonic function) are scaled according to material permittivities~\cite{jackson1999}.
Consequently, given two domains $\Omega_1$, $\Omega_2$ with permittivities of $\epsilon_1$ and $\epsilon_2$ and (quantum) neural networks $\phi_1(\bm x;\bm \theta_1)$ and $\phi_2(\bm x; \bm \theta_2)$ defined on each domain respectively, we construct a normal unit vector $\bm n(\bm x)$ to the boundary $\partial\Omega$ between $\Omega_1$ and $\Omega_2$ and define the following loss term:

\begin{equation}\label{eq:dielectric_loss}
\begin{split}
   L_{\textrm{PI}_{\textrm{DE} }}(\bm\theta_1, \bm\theta_2) = &\quad E_{\bm x\sim\partial\Omega} \left[\left(\phi_1(\bm x;\bm \theta_1) - \phi_2(\bm x;\bm \theta_2)\right)^2 \right. \\ 
   &\quad + \left. \left(\epsilon_1 \pd{E_1}{\bm n} -  \epsilon_2 \pd{E_2}{\bm n} \right)^2 \right]
\end{split}
\end{equation}
where $E_i$ represents the electric field in domain $\Omega_i$. In general contexts, we have $E_i = \nabla \phi_i (\bm x, \bm \theta)$. However, in the case of our curl-based approach, we represent $E_i$ by the underlying divergent free field $\nabla \times A$ as outlined in Sect. \ref{theory:curl}.

\subsubsection{Neural Networks}

\paragraph{(q)PINN (baseline)} We define a single (quantum or conventional) neural network, $\phi_{\textrm{PINN}}(\bm x; \bm \theta)$ and minimise the loss
\begin{equation}\label{eq:pinn_loss}
L_{\textrm{PINN}} = L_{\textrm{PI}_H} + L_d
\end{equation}

with $L_{\textrm{PI}_H}$ and $L_d$ defined in Eqs. \ref{eq:dirichlet_loss} and \ref{eq:laplacian_loss} respectively.

\paragraph{h(q)PINN (baseline)}
Given a multilayer perceptron $\textrm{MLP}(\bm x; \bm \theta)$ and a Dirichlet boundary $\Gamma$ taking constant value $c \in \mathbb{R}$ on it, we can construct a neural network that exactly satisfies the boundary condition while leaving flexibility to optimise towards harmonic functions~\cite{lu2021}. First, we define a distance function $d(\bm x)$ representing the shortest distance between a generic point $x\in \Omega$ and $\Gamma$, and then define a neural network of the form:
\begin{equation}\label{eq:hpinn_def}
    \phi_{\textrm{hPINN}} = c e^{-k d(\bm x)} + (1 - e^{-k d(\bm x)}) \textrm{MLP}(\bm x; \bm \theta),
\end{equation}
where $k \in \mathbb{R}$ is a hyperparameter to be chosen. Since this network satisfies Dirichlet condition automatically, it can be trained by optimising on Eq. \ref{eq:laplacian_loss} directly.

\paragraph{(q)Holomorphic (ours)}
We construct a harmonic neural network as defined in Eq. \ref{eq:harmonic_simply_connected} and a harmonic quantum neural network as defined in Eq. \ref{eq:harmonic_simply_connected_qnn}. Since these networks are harmonic by construction, our task of training them is reduced to a supervised learning problem on the boundary terms/data. Consequently, we train them by minimising Eq. \ref{eq:dirichlet_loss} alone.

\paragraph{Curl(q)Net (ours)}
We define two neural networks $\phi_c(\bm x, \bm \theta_1)$ and $A(\bm y, \bm \theta_2)$ as outlined in Sect.\ref{theory:curl}. 
Note that  $\nabla \phi_c$ and $\nabla \times A$ must operate in the same space, however in general the two (q)NNs can be defined on different spaces.

The Curl(q)Net solution is achieved by minimising the following objective function:
\begin{equation}\label{eq:curlnet_loss}
    L_{\phi_c} = L_d +  L_c
\end{equation}
with $L_d$ and $L_c$ defined in Eqs. \ref{eq:dirichlet_loss} and \ref{eq:curl_net_loss} respectively. 
Indeed, satisfying $L_c$ amounts to satisfy $L_{\textrm{PI}_H}$. 
We substitute Eq. \ref{eq:curl2d} into \ref{eq:curl_net_loss} in our scenario since our applications are in two dimensions. 
We further note that the term $L_d$ can be dropped from the optimisation procedure if $\phi_c$ is defined in a manner analogously to Eq. \ref{eq:hpinn_def}.

\paragraph{Multiholomorphic (ours)}
In the case of multiply-connected domains, we decompose the domain into simply-connected components so that we can still leverage holomorphic function to derive harmonic neural networks (see also Fig. \ref{fig:triangular_heater}a). Thus, we take a multiholomorphic network $\phi_{\textrm{MH}}$ as defined in Eq. \eqref{eq:harmonic_multiply_connected} and minimise the following:
\begin{equation}\label{eq:multiholomorphic_loss}
   L_{{MH}} =  L_{\partial\omega} + L_d
\end{equation}
with $L_d$ and $L_{\partial\omega}$ defined in Eqs. \ref{eq:dirichlet_loss} and \ref{eq:multiply_connected_loss} respectively.

\paragraph{XPINN (baseline)}
Since our multiholomorphic function involves partitioning a domain, we include a domain decomposition strategy applied to PINNs as a further benchmark to try and maintain a fair perspective. The fundamental idea behind XPINNs is to decompose a domain into disjoint subdomains ~\cite{jagtap2020}, whilst variationally inducing the continuity of the solutions across each subdomain. We can define a neural network $\phi_{\textrm{XPINN}}(\bm x; \bm \theta)$ analogously to the multiholomorphic network in Eq. \ref{eq:harmonic_multiply_connected}, except we use real-valued multilayer perceptrons on each subdomain as opposed to harmonic neural networks.
Consequently we can train $\phi_{\textrm{XPINN}}(\bm x; \bm \theta)$ by minimising
\begin{equation}
\label{eq:multiholomorphic_loss2}
   L_{\textrm{XPINN}} =  L_{\partial\omega} + L_d + L_{\textrm{PI}_H},
\end{equation}
with $L_{\partial\omega}$,  $L_d$ and $L_{\textrm{PI}_H}$ defined in equations \ref{eq:multiply_connected_loss},  \ref{eq:dirichlet_loss} and \ref{eq:laplacian_loss} respectively, the latter contribution being necessary due to dropping of conditions imposing harmonicity.

\subsection{Experimental Setup}\label{sec:experimental_setup}
We conduct all (both quantum and conventional) experiments in Python3 and make use of NumPy~\cite{harris2020}, PyTorch~\cite{paszke2019} and Matplotlib~\cite{hunter2007} packages. FEM ground truths were constructed using FEniCS~\cite{alnaes2015}. All the QNNs used in this paper are implemented with proprietary code, leveraging upon the packages PyTorch and Yao.jl~\cite{luo2020yao}. 

We implement all classical neural networks in PyTorch~\cite{paszke2019}. We consistently use multilayer perceptrons with 3 hidden layers, width 32, initialised with Kaiming Uniform initialisers~\cite{he2015}, optimised for 16,000 epochs over full-batches with an Adam optimizer \cite{kingma2014}  with a learning rate (LR) of $10^{-3}$. We use $\tanh$ activations for real-valued NNs and $\sin$ activations for holomorphic NNs. 
All the QNNs used in this paper are implemented with proprietary code, leveraging upon the packages PyTorch and Yao.jl~\cite{luo2020yao}. 
Further details, including details on simulated quantum circuits and hardware specifics, are in Appendix \ref{app:exp_details}. Here, however,  we observe that a QNN defined over an $n$-qubit quantum circuit has an expressivity comparable to a spectral expansion employing $2^n$ terms. This is empirically observed e.g. in Sect.~\ref{sec:application_heat2}.
Such an exponential increase in the expressivity with the quantum circuit size holds promise for potential future advantage in expressing targeted solutions.

For each application, we construct boundary conditions comprising of lines where we sample 100 uniformly-spaced points. To minimise physics-informed losses, we sampled 1024 collocation points randomly (uniformly) on the interior of each domain. These points are sampled once and then kept constant for each experiment to allow each run the same amount of information.

\begin{figure}
  \begin{center}
    \includegraphics[width=0.3\textwidth]{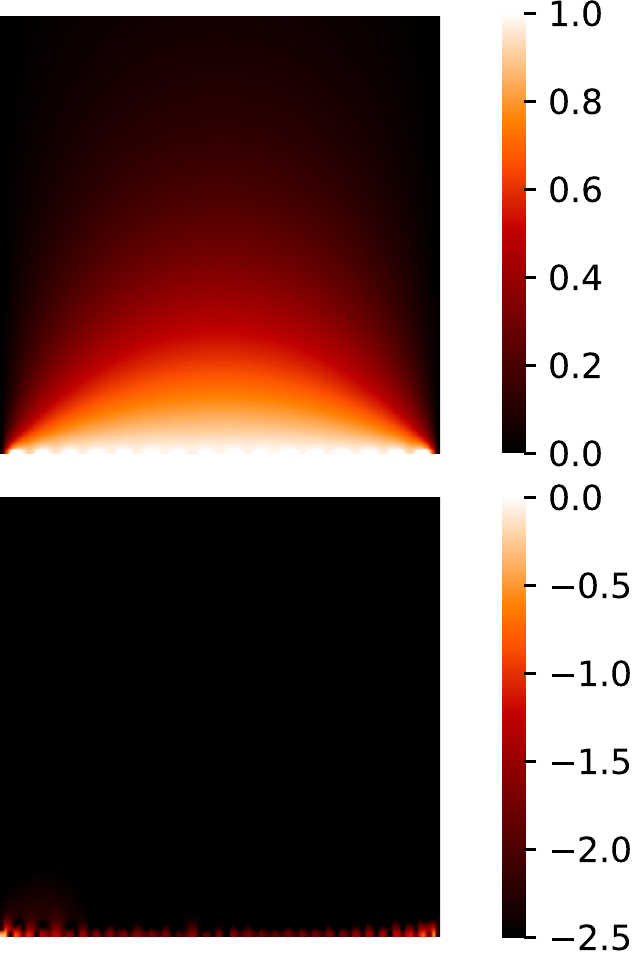}
  \end{center}
  \caption{(top) heat profile in arbitrary units for a refrigerated box with the bottom side heated, computed with a numerical simulation of a 4 qubits hQNN. (bottom) log\textsubscript{10}-scale error of hQNN vs analytical solution.}
  \label{fig:harmonic_qnn_result}
\end{figure}

\subsection{Results A: Dielectric Material in a Charged Box}\label{sec:application_electrostatic}

Consider a dielectric material placed in a 2D box grounded on 3 sides, and an electric potential applied to the bottom edge. 
We consider the task of inferring the underlying electric potential field within the box, in the presence of different materials.

We thus introduce a first dielectric domain $\Omega_1 = [0, 1] \times [0, 0.5]$, with permittivity $\epsilon_1 = 1.0$ and free-space above the dielectric as $\Omega_2 = [0, 1] \times [0.5, 1]$ with permittivity $\epsilon_2 = 0.01$. The bottom lid ($y=0$) has an applied voltage of $V=1.0$ (in arbitrary units), whereas $V=0.0$ on all other boundaries.
We define neural networks and losses on each $\Omega_1$ and $\Omega_2$ as outlined in Sect. \ref{sec:methods}. 
In addition, we couple the losses of both networks using Eq. \ref{eq:dielectric_loss} and optimize both networks jointly. 
We do not consider our multiholomorphic or the XPINN architecture in this scenario, since they were devised for use on multiply-connected domains, as opposed to the simply-connected setting in this application.

We report in Table \ref{tab:benchmarks} the results as (i) root mean squared errors (RMSE) against FEM solutions, i.e.
$\sum_{\bm{x}_i \in \tilde{\Omega}} \sqrt{
\left(\phi_{}(\bm{x}_i) - \phi_{\textrm{FEM}}(\bm{x}_i)\right)^2
}/|\tilde{\Omega}|$, where $\tilde{\Omega} \equiv \{\bm{x}_i\}$ a set of 
collocation points uniformly sampled in the interior domain $\Omega$, along with (ii) the expected absolute Laplacian over $\Omega$.

We note that PINNs achieve a low mean Laplacian over the domain, even though they do not represent the best-performing solutions. One possible explanation would be that the Laplacian is directly optimised for PINNs (equations \ref{eq:laplacian_loss} and \ref{eq:pinn_loss}), but minimising such a loss could still lead to systematic errors across a spatial domain.

\begin{table*}
  \caption{Performance of our methods on the electrostatic, heat-distribution and navigation benchmarks illustrated respectively in Sect. \ref{sec:application_electrostatic}, \ref{sec:application_heat} \& \ref{sec:application_robot}. Note that all numbers are multiplied by 100 for clarity.  We present the mean and standard deviation of each method over 10 repeated runs. In addition, we report the expected absolute Laplacian over the spatial domain, which should be identically zero for perfectly harmonic functions. We separate quantum (see also Fig. \ref{fig:triangular_heater}) and conventional techniques with a single horizontal line. Lower numbers indicate better performance. Each best-performing method is highlighted by bold text.
  }
  \label{tab:benchmarks}
  \centering
\resizebox{\textwidth}{!}{
\begin{tabular}{lcccccccc}
\toprule
& \multicolumn{2}{c}{Electrostatics}& \multicolumn{2}{c}{Heat Distribution} & \multicolumn{2}{c}{Robot Navigation} & \multicolumn{2}{c}{Fluid Flow} \\
& RMSE & $E\left[\left| \nabla^2 \phi \right|\right]$ & RMSE & $E\left[\left| \nabla^2 \phi \right|\right]$ & RMSE& $E\left[\left| \nabla^2 \phi \right|\right]$ &RMSE &$E\left[\left| \nabla^2 \phi \right|\right]$\\
\midrule
Holomorphic (ours) & \textbf{0.3 \ppm 0.09} & 0.0 \ppm 0.0 & 19.1 \ppm 0.05 & 0.0 \ppm 0.0 & 17.1 \ppm 0.5 & 0.0 \ppm 0.0 & - & -\\
CurlNet (ours) & 1.4 \ppm 0.06 & 165 \ppm 27.3 & 2.4 \ppm 0.004 & 0.04 \ppm 0.05 & \textbf{1.6 \ppm 1.1} & 2013 \ppm 1502 & \textbf{11.2 \ppm 0.7} & 162 \ppm 8.3\\
PINN & 6.7 \ppm 1.8 & 5.3 \ppm 0.6 & 2.4 \ppm 0.0006 & 0.002 \ppm 0.0003 & 24.2 \ppm 28.1 & 11.2 \ppm 1.8 & 18.1 \ppm 0.6 & 10.9 \ppm 5.9\\
hPINN & 10.0 \ppm 0.03 & 11.4 \ppm 1.3 & 2.4 \ppm 0.006 & 0.03 \ppm 0.05 & - & - & 14.6 \ppm 0.08 &  91.4 \ppm 33.7\\
Multiholomorphic (ours) & - & - & \textbf{0.2 \ppm 0.09} & 0.0 \ppm 0.0 & - & - & - & -\\
XPINN & - & - & 2.4 \ppm 0.02 & 0.1 \ppm 0.3 & - & - & - & -\\
\midrule
CurlqNet (ours) &  \textbf{3.4 \ppm 0.03} & 269 \ppm 15.4 & \textbf{2.1 \ppm 0.02} & 2.8 \ppm 0.4 & \textbf{6.6 \ppm 0.09} & 307 \ppm 9.4 & - & - \\
qPINN &  13.8 \ppm 0.3 & 8.4 \ppm 0.7 & 17.2 \ppm 0.3 & 28.2 \ppm 2.4 & 33.0 \ppm 0.4 & 14.3 \ppm 0.6 & - & -\\
hqPINN   &  16.2 \ppm 1.8 & 10.1 \ppm 19.7 & 6.5 \ppm 0.04 & 2.4 \ppm 0.5 & - & - & - & -\\
qHolomorphic (ours)  &  - & - & - & - & 16.5 \ppm 0.8 & 0.0 \ppm 0.0 & - & -\\
\bottomrule
\end{tabular}
}
\end{table*}

\subsection{Results B: Heat Distribution in a Box with Single-Sided Heating}\label{sec:application_heat2}
We consider a refrigerated square box with uniform heating on a single side, where we wish to find the steady-state heat distribution described by $\phi$. We consider a domain $\Omega = [0, 1] \times [0, 1]$. Desired Dirichlet boundary conditions are $\phi=1$ on the edge $x=0$, and $\phi=0$ (in arbitrary units) when $y=0$, $y=1$, or $x=1$. Note that this problem is a single (simply-connected) domain. We know the analytical solution as an infinite series, and plot the converged result when adopting a \textit{qHolomorphic} architecture, along with the absolute error with the analytical result, in Fig. \ref{fig:harmonic_qnn_result}.
We use the availability of an analytical result for this case to highlight the expressivity of QNN architectures in comparison with classical spectral solutions, as discussed in detail in Appendix~\ref{app:quantum_laplace}).

\subsection{Results C: Heat Distribution Around a Heater}\label{sec:application_heat}
We consider a heated triangular body in a refrigerated box, where we wish to find the steady-state heat distribution described by $\phi$.

We consider a domain $\Omega = [0, 10] \times [0, 10]$, with an equilateral triangle of length $4$ centred in the box as illustrated in Fig.  \ref{fig:triangular_heater}a.
Desired Dirichlet boundary conditions are $\phi=1$ on the triangle edges and $\phi=0$ (in arbitrary units) when $x=0$, $x=10$, $y=0$ or $y=10$. With these conditions, $\phi$ defines a harmonic function representing the steady-state temperature distribution inside the box.

\begin{figure}
\centering
\includegraphics[width=\columnwidth]{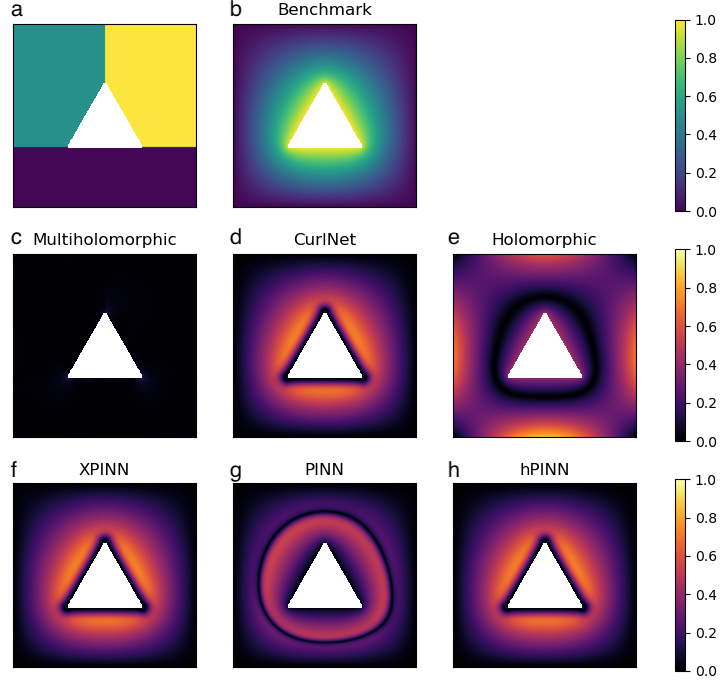}
\caption{
(a) A partitioning of the multiply-connected domain in Sect.\ref{sec:application_heat} into three simply-connected domains on which we define our multiholomorphic and XPINN networks. Note that the white triangle is not considered part of the domain and is excluded from the analysis.
(b) The solution $\phi(x,y)_{\textrm{FEM}}$ of the heater problem described in Sect.~\ref{sec:application_heat}, as provided by a FEM solver for benchmark.
(c-h) A comparison of the performance for different neural network solutions, by plotting the absolute error of each network compared against $\phi(x,y)_{\textrm{FEM}}$ given respectively by (c) a Multiholomorphic neural network, (d) a CurlNet architecture, (e) a Holomorphic network, (f) an XPINN domain decomposition based approach, (g) a regular PINN and (h) an hPINN architecture with Dirichlet boundary conditions included in the network architecture. 
Additional details and definitions are provided in the main text. 
All plots report a sample of the reference metric on a uniform $512\times 512$ grid. 
\textit{Results for the same test case adopting various quantum architectures are reported in App.~Fig.~\ref{fig:dqc_comparative_performances}}.
}
\label{fig:triangular_heater}
\end{figure}

Note that this problem is not on a simply-connected domain. 
For example, a closed path around the triangular body cannot be reduced to a point, without crossing the body. Consequently, we have reason to believe that the resulting solution will not be representable by our holomorphic formulation, though it could be representable by a multiholomorphic formulation. 
We partition our domain as outlined in Fig. \ref{fig:triangular_heater}a for both XPINN and multiholomorphic approaches, and report our results in Table \ref{tab:benchmarks}.
Note that as in section \ref{sec:application_electrostatic}, we note that a low expected absolute Laplacian does not necessarily correspond to a good solution in terms of RMSE.
All the figures of merit for this case are equivalent to those outlined in Sect. \ref{sec:application_electrostatic}.

\subsection{Results D: Robot Navigation in a Previously Explored Environment}\label{sec:application_robot}

We consider a robot navigating a known, static environment, e.g. a pre-explored corridor. The use of general potential to guide robot movement is well-established, and it can present evident advantages over traditional path-finding algorithms \cite{rimon1990exact}. For example, it offers a symbolic, parameterized representation of the environment that can be easily updated with new information, as well as exploiting special properties of the mechanical system to ensure more stringent adherence to problem constraints, such as e.g. a maximum available torque to the robot.   
In particular, it has been noted that harmonic functions have favourable properties for robot navigation \cite{connolly1993applications}: harmonic functions do not attain local minima anywhere (except for constant harmonic functions), which prevents navigation paths from becoming stuck in such minima. This property led to a few dedicated studies \cite{kazemi2004robotic, loizou2014multi}.

Define a domain $\Omega = ([0, 0.5] \times [0, 0.6]) \cup ([0.5, 1.0] \times [0.4, 1.0])$.
We construct a harmonic function, $\phi$ with Dirichlet boundary conditions  $\phi=-1$ at $y=0$ and $\phi=1$ at $y=1$. A summary of resulting paths and errors in the potential relative to finite element benchmarks can be found in Fig. \ref{fig:robot_analysis}.

\begin{figure}[ht!]
\includegraphics[width=\columnwidth]{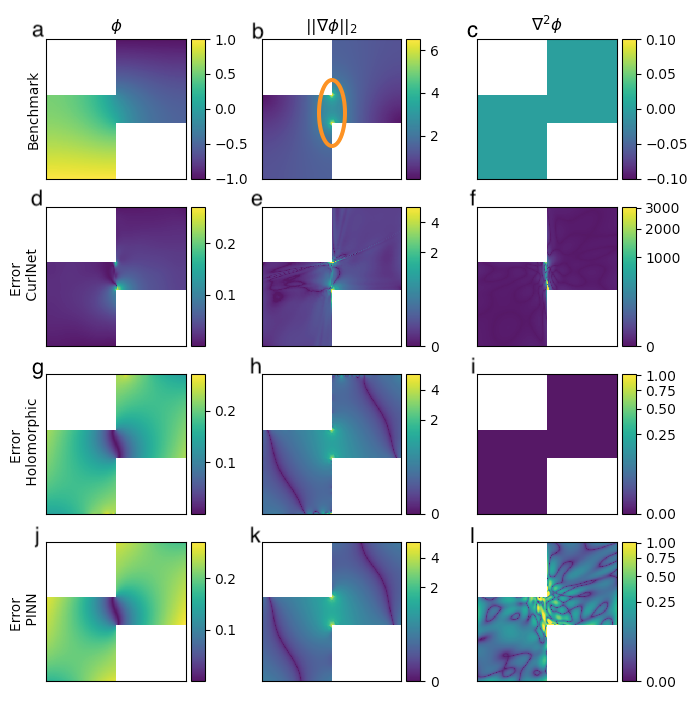}
\caption{
Benchmark and sample errors from the CurlNet, Holomorphic and PINN approaches in the robot navigation task. The orange oval in subplot b highlights an area of rapidly changing conservative field, which can be difficult for models to fit effectively. (a) Baseline potential (i.e. $\phi$) solution generated by finite elements. (b) Baseline magnitude of conservative field (i.e. $\|\nabla\phi\|_2$) calculation generated by finite elements (c) The exact reference Laplacian (i.e. identically zero). 
(d) The absolute error in the potential from a sample CurlNet run. (e) The absolute error in the magnitude of the conservative field in the CurlNet solution. (f) The absolute error in the Laplacian of the CurlNet solution. 
(g) The absolute error in the potential of the holomorphic solution. (h) The absolute error in the magnitude of the conservative field in the holomorphic solution. (i) The absolute error in Laplacian in the holomorphic solution. 
(j) The absolute error in the potential of the PINN solution. (k) The absolute error in the magnitude of the conservative field in the PINN solution. (l) The absolute error in Laplacian in the PINN solution. 
Please note the log colour scale of subplot f, showing very high Laplacians in a small central region.
\textit{Results for the same test case adopting various quantum architectures are reported in App.~Fig.~\ref{fig:classical_robot_results_fig}}.
}
\label{fig:robot_analysis}
\end{figure}

We note that even though the holomorphic network did not characterise the underlying potential field as accurately as the CurlNet, it provided the most consistent robot navigation paths.
This might be explained by the high errors in Laplacian measured in the CurlNet approach.

That CurlNet offers a more accurate potential despite having a higher Laplacian is noteworthy. In this particular spatial domain, we note that there are rapidly changing conservative fields (see Fig.\ref{fig:robot_analysis}b. The CurlNet approach is able to match this behaviour well (Fig.\ref{fig:robot_analysis}e), but at the expense of introducing order of magnitude increases in its Laplacian (Fig. \ref{fig:robot_analysis}f). 
However, holomorphic and PINN approaches directly constrain a for a zero Laplacian ($\frac{\partial^2\phi}{\partial x^2} + \frac{\partial^2\phi}{\partial y^2} = 0$). Evidently, optimising for a zero Laplacian but a large gradient is a difficult optimisation task. Consequently, in order to maintain a low Laplacian, we note that the conservative field in holomorphic and PINN approaches show high errors (Fig.\ref{fig:robot_analysis}h and k). In order to maintain low Laplacians in these regions

See Table \ref{tab:benchmarks} for a full comparison of mean squared errors in the potentials and Laplacian's over valid points (i.e. within the environment accessible to the robot) extracted from a uniform $64 \times 64$ grid, for both classical methods and their previously introduced quantum counterparts.

\subsection{Results E: Potential Flow Through a 3D Pipe}
\label{app:results_3Dpipe}

We consider fluid flowing through a square pipe and study potential flow, denoted by $\phi$.
We take the domain to be the unit cube: $\Omega = [0, 1]^3$ with a boundary that we denote by $\partial\Omega$. 
As boundary conditions, we impose $\phi(0.0, y, z) = 1.0$, $\phi(1.0, y, z) = -1.0$ and $\phi(x, y, z) = 0.0$ everywhere else on $(x,y,z) \in \partial\Omega$.

We benchmark the classical methodologies suitable for problems in dimensions higher than two on this system, where we find favourable performance relative to PINNs, as outlined in Table~\ref{tab:benchmarks}.

\section{Discussion}
In this work we demonstrated the inclusion of inductive biases in classical and quantum neural networks alike towards harmonic functions, observing how despite their crucial importance, there had been only few attempts at including specifically harmonic inductive biases. 
In summary, we constructed exactly harmonic functions in two dimensions in simply-connected domains, harmonic functions almost everywhere in two-dimensional multiply-connected domains, and approximately harmonic functions in arbitrary domains.
We demonstrated our approach with comparisons to (quantum) physics-informed neural networks \cite{raissi2019, kyriienko2021dqc} in a range of tasks, namely heat distribution, electrostatics and robot navigation.
Previous tools in this domain were mostly limited to finite-element modelling \cite{galybin2010}, which is less suited for problems such as inverse design and transfer learning.
Training upon conventional neural networks as well as simulated quantum circuits confirmed the benefits of our newly proposed architectures.
Interestingly, we note that achieving a low Laplacian in a given solution is not always consistent with achieving low mean squared errors against a trustworthy benchmark. 
It is perhaps unsurprising that PINNs typically have low differential equation residuals when trained since they optimise against PDE residuals directly.
However, it is interesting that this might not necessarily be the best strategy for achieving a low error in the resulting solution, as evidenced by the fact that our CurlNet architectures consistently achieve lower RMSEs than their PINN counterparts, whilst exhibiting higher errors in their Laplacians, which we explain further in Sect.\ref{sec:application_robot}.
Besides the merits in all those applications where harmonic functions are relevant, our work highlights how developments in conventional machine learning architectures can be sometimes readily ported to the corresponding realm of quantum machine learning, here exemplified by trainable, parameterised quantum circuit architectures. 
Potential future work might include extending holomorphic functions to higher dimensions. We note that there are connections between so-called regular quaternionic functions and harmonic functions that might facilitate such advancements \cite{sudbery1979quaternionic}.
We hope that our findings encourage further cross-pollination between the quantum and classical machine learning fields.


\section{Acknowledgements}
This work has benefitted from feedback and suggestions for improvement from anonymous ICML reviewers. We would like to thank them for their time and effort.

\bibliography{references}
\bibliographystyle{icml2023}

\newpage
\appendix
\onecolumn


\section*{Appendix}

\section{Details on Quantum Approaches}

\subsection{Further Details about (Inductive Biases for) Quantum Algorithms}
\label{app:qalgorithms}

Quantum computing is a computational paradigm that
holds theoretical promise for exponential speedups over conventional computers in, for example, factoring prime numbers~\cite{shor1994} and solving linear systems of equations~\cite{harrow2009}.
However, the practical application of such provably advantageous algorithms is hindered by the lack of fault-tolerant quantum computers, which are still thought to be years away in terms of development.
In spite of such theoretically-guaranteed computational advantages being limited to future devices, there have already been indications of experimental quantum advantage having been achieved for certain restricted classes of problems~\cite{arute2019, zhong2020}.
While such experimental advantages have been demonstrated, it remains to find advantages in applications which have wider practical application, leading to the development of the variational quantum algorithms~\cite{cerezo2021} and QNNs discussed in the main paper. In particular, QNNs have been used for classification and regression tasks~\cite{Mitarai2018}, solving differential equations~\cite{kyriienko2021dqc, Paine2021}, model discovery~\cite{Heim2021} and extremal learning~\cite{Varsamopoulos2022}.  

These quantum NNs (QNNs) hold the promise of being advantageous in terms of the \textit{expressivity}~\cite{schuld2019quantum,abbas2021}, offered by parameterised quantum circuits, whereas their \textit{trainability} is susceptible to issues~\cite{mcclean2018, anschuetz2022beyond}, e.g. large regions of parameter space exhibit exponentially-vanishing parameter gradients with respect to loss functions, so-called `barren plateaus'.   

One successful means to mitigate barren plateaus has been to include inductive biases within quantum circuit designs~\cite{hadfield2019, bharti2021}.
It has even been argued that, in some settings, including inductive biases in quantum circuits is a prerequisite for quantum advantage~\cite{kubler2021}.
Including inductive biases in conventional NNs has been a staple of progress in machine learning, e.g. ranging from convolutional neural networks~\cite{lecun1998} and geometrically-invariant networks~\cite{giles1987} through to recent examples such as Hamiltonian neural networks~\cite{greydanus2019} and physics-informed neural networks~\cite{raissi2019}.
However, the use of inductive biases in designing quantum circuits and quantum machine learning models is still a nascent field, though there is a growing research interest on this front~\cite{mernyei2021, larocca2022}. 
the  though there is a growing research interest on this front~\cite{mernyei2021, larocca2022}. 
In the context of variational quantum algorithms, inductive biases are often referred to as problem-inspired ans{\"a}tze.
A well-established example is the unitary coupled cluster ansatz for modelling the ground-state energy of molecular Hamiltonians~\cite{lee2018}. 
Hard constraints in ans{\"a}tze for combinatorial optimization have also been previously reported~\cite{hadfield2017}.
Analogously to how neural architecture search~\cite{elsken2019} might learn architectures with given inductive biases, there have been developments on adaptively constructing ans{\"a}tze for molecular simulations~\cite{grimsley2019}.
Another well-known paradigm concerning the adoption of inductive biases in the context of variational quantum algorithms refers to the characterisation of quantum systems. In this case, learning a full dynamic map describing a quantum process (i.e. performing ab initio a quantum tomography) is unfeasible already for systems of very few qubits~\cite{mohseni2008qpt}.
However, the system dynamics can be modelled constraining such maps by knowledge of (parametrised) generators of the dynamics (e.g. the system's Hamiltonian~\cite{Granade2012qhl}) and of their eigenstates~\cite{greiter2018method}. Similarly to aforementioned ans{\"a}tze in molecular simulations, such parametrised generators can also be made adaptive~\cite{gentile2021learning}.

\subsection{Holomorphic Quantum Circuits and Harmonic Quantum Neural Networks}
\label{app:qholomorphic}
Our aim is to construct quantum circuits which exhibit holomorphic properties.\\
In a general form, a QNN model can be written as 
\begin{equation}\label{most_general_qnn}
        \phi_{Q}(\bm x) = \langle \emptyset | \hat{\mathcal{U}}(\bm x, \bm \theta)^\dagger \hat{\mathcal{H}}(\bm x,\bm \theta) \hat{\mathcal{U}}(\bm x, \bm \theta) | \emptyset\rangle,
\end{equation}
where the unitary circuit ansatz $\hat{\mathcal{U}}(\bm x, \bm{\theta})$, and the Hermitian cost Hamiltonian $\hat{\mathcal{H}}(\bm x,\bm{\theta})$, are functions of function-input variables and model parameters $\bm{\theta}$, and $|\emptyset\rangle$ represents the 0-state, i.e. a reference state for the quantum computation. Alternatively, a different type of quantum model can be written as
\begin{equation}
    \phi_{QH}(\bm x) = \Re{\langle \emptyset | \hat{\mathcal{U}} (\bm x,\bm{\theta}) | \emptyset\rangle},
\end{equation}
which represents the real part of a complex number obtained from measuring the overlap of the parametrized state (i.e. the reference 0-state evolved by the unitary $\mathcal{\hat U}$, for example prepared on a register of qubits), with the 0-state itself. Such overlap can be computed using e.g. the Hadamard or SWAP tests \cite{nielsen2002quantum}.\\ 
To compute the value of $\phi_Q(\bm x)$ or $\phi_{QH}(\bm x)$, or even approximate it, is believed to be very computationally intensive on a classical computer for systems involving large unitaries $\mathcal{\hat U}$ (which in turn corresponds to the size of the target system). On the contrary, this is typically feasible on a quantum computer, provided some assumption on the circuit structure, and eventually the cost Hamiltonian for $\phi_Q(\bm x)$ \cite{arute2019}.\\
For simplicity and ease of analysis, we here restrict ourselves to a specific structure of the cost Hamiltonian as being independent of $\bm x$ and $\bm{\theta}$, i.e. $\hat{\mathcal{H}}(\bm x,\bm{\theta})=\hat{\mathcal{H}}$, and the circuit structure as 
\begin{equation}\label{simpler_unitary_qnn}
    \hat{\mathcal{U}}(\bm{x}, \bm \theta) = \hat{\mathcal{U}}_{\theta_2}\textrm{QFM}(\bm x)\hat{\mathcal{U}}_{\theta_1},
\end{equation}
is then a single quantum feature map $\textrm{QFM}(\bm x)$, squeezed between two general variational circuits $\hat{\mathcal{U}}_{\theta_j}$ \cite{kyriienko2021dqc} which only depend on model parameters $\bm{\theta}_j$.

To describe a particular type of QFM in the form of a Hamiltonian evolution, we first make use of the following decomposition of a Hamiltonian evolution applied to some arbitrary initial state $|\Psi_{\text{ini}}\rangle$:
\begin{equation}\label{eq:hamiltonian_evolution}
e^{- i t\hat{\mathcal{G}}}|\Psi_{\text{ini}}\rangle = \sum_m c_m e^{- i t E_m}|\psi_m\rangle.
\end{equation}
For the examples in our paper  $\bm x \equiv (x,y)$ - as we only deal with 2D problems - and for a QFM of the form in \eqref{eq:hamiltonian_evolution}, we can think that the dependency upon the two variables can be expressed via the time: 
\begin{equation}
    t = t(x,y).
    \label{eq:timefunction}
\end{equation}
Inserting this expression in \eqref{most_general_qnn}, where $|\Psi_{\text{ini}}\rangle \equiv |\emptyset\rangle$, using the assumption \eqref{simpler_unitary_qnn}, and reminding that the decomposition in \eqref{eq:hamiltonian_evolution} involves orthonormal $|\psi_m\rangle$, which are also eigenstates of $\mathcal{\hat G}$, i.e. $\mathcal{\hat G} |\psi_m\rangle = E_m |\psi_m\rangle$, we find:
\begin{equation}\label{eq:phiq}
\phi_{Q}(x,y) = \langle \Psi(\theta_1)| e^{i t \mathcal{\hat G}}
\hat{\tilde{\mathcal{H}}}(\theta_2) 
e^{- i t \mathcal{\hat G}} |\Psi(\theta_1)\rangle = \sum_{\langle m,n\rangle} c_{mn}(\theta_1, \theta_2) e^{- i t (E_m-E_n)}
\end{equation}
where $\hat{\tilde{\mathcal{H}}}= \hat{\mathcal{U}}_{\theta_2}^\dagger \mathcal{\hat H} \hat{\mathcal{U}}_{\theta_2}$ indices $\langle m,n\rangle$ sum over all $m$ and $n$ from $1$ to the size of the accessible Hilbert space (i.e. $2^N$ for an N-qubit register). 
While $c_{mn}$ are generally complex-valued, the resulting function output is guaranteed to be real-valued.\\ 
For a QNN of the form $\phi_{QH}(x,y)$ we find
\begin{equation}\label{eq:phiqh}
\phi_{QH}(x,y) = \Re\{\sum_m c_m(\theta_1)  \langle \emptyset|\hat{\mathcal{U}}_{\theta_2} |\psi_m\rangle e^{- i t E_m}\} = \Re\{\sum_m \tilde{c}_m(\theta_1, \theta_2) e^{-i t E_m}\}.
\end{equation}
If we now consider a time evolution in \eqref{eq:timefunction} of the form $t(x,y) = y - x i$, we obtain the result from the main text in Eq.~4. This yields a sum over exponentials of holomorphic functions, which is thus holomorphic itself and has a harmonic real part.\\
With this choice, we find the quantum models to be represented by linear combinations of $e^{-(x+iy) E}$ with $E$ some number representing an eigenvalue, or gap between two eigenvalues, of the QFM evolution Hamiltonian. The number of unique terms is therefore equal to the number of unique (non-degenerate) eigenvalues or gaps in the spectrum of the QFM evolution Hamiltonian. The more unique terms, the more basis functions and thus the more expressivity the quantum model can obtain. We describe here the specific complex-exponential QFM we chose for the results section
\begin{equation}\label{eq:qfm}
\text{QFM}(x,y) = e ^ {(x-iy)\pi\hat{\mathcal{H}}} \hspace{10pt} \text{where} \hspace{10pt} \hat{\mathcal{H}}=\sum_{j=0}^{N-1} 2^j \hat{Z} + 2^N.
\end{equation} 
this Hamiltonian has $2^N$ unique eigenvalues $E_m = 2m-1 =\{1,3,5,7\ldots(2\cdot2^N - 1)\}$.

\subsection{Applying Harmonic QNNs to Solve an Instance of the Laplace Equation}
\label{app:quantum_laplace}

In the main paper, Sect. 4.4 (Results B), we considered the solution of the Laplace equation for the heat distribution $f(x,y)$ inside a unit-square box refrigerated on all sides (where then the Dirichlet b.c. applicable impose $f|_{\partial \Omega_{1-3}} = 0$) but one (occurring at $x=0$), where we assume the presence of a heater, modelled by imposing a Dirichlet b.c. $f|_{\partial \Omega_{4}} = 1$. \\
The general solution to this problem setting is known in closed form as a series expansion 
\begin{equation}
    f(x,y) = \sum_{n \text{ odd}}^{\infty} \frac{4}{n\pi} e^{-n \pi x} \sin{(n \pi y)} = \frac{2}{\pi} \arctan{\frac{\sin(\pi y)}{\sinh{\pi x}}}
    \label{eq:laplace_square_expansion}
\end{equation}
The exact solution can thus also be numerically approximated; as the coefficients become smaller with the term index, by truncating the sum to a given number of terms $\tilde{N}$, the solution becomes closer and closer to exact. We note that expansions equivalent to \eqref{eq:laplace_square_expansion} exist also for cases where $f|_{\partial \Omega_{4}} = f(y)$, and for the general case of a 2D Laplace equation with Dirichlet b.c., it is expected that alternative expansions can naturally approximate the solution.
However, in general, the analytical function representation of the series is not always known, and therefore numerical truncation methods are the best alternative. Having access to a large number of terms improves the approximation.\\

In Fig.~\ref{fig:SI_plot_qnn} we plot the converged results and the loss convergence profile for a QNN adopting $N=4$ qubits. We find the solution matches very well overall, with only a slight deviation at the $x=0$ boundary, for all $y$, with clear oscillatory behaviour. This can be explained by the so-called Gibbs phenomenon, which is a demonstration that step-function behaviour can in principle be modelled accurately with periodic functions, but the amplitude error remains significant even for a larger cut-off in the series. The step-function is required by the sharp transition $\lim_{y \rightarrow 0^+} f(x=0,y) =1$ whilst $\lim_{y \rightarrow 0^-} f(x=0,y) =0$, which occurs at the corners of the unit-box for $x=0$ due to the discontinuous b.c. imposed.

We observe very similar profiles (results omitted) for the solution represented by a finite sum, with a cut-off at 16 terms.
This confirms the intuition about the qHolomorphic circuit expressivity outlined in the last paragraph: the $N=4$ case employs $2^N=16$ terms so that we expect a performance similar to expressing the solution by introducing the same number of terms in \eqref{eq:laplace_square_expansion} - where the coefficients of the terms are fixed and not variationally optimised. 
This intuitive reasoning signifies the capacity of the qHolomorphic circuit to represent basis functions to the Laplace problem.\\
Finally, we stress how the qHolomorphic architecture converges very rapidly to the approximate solution, in less than 100 epochs.

\begin{figure}
    \centering
    \includegraphics[width=0.7\linewidth]{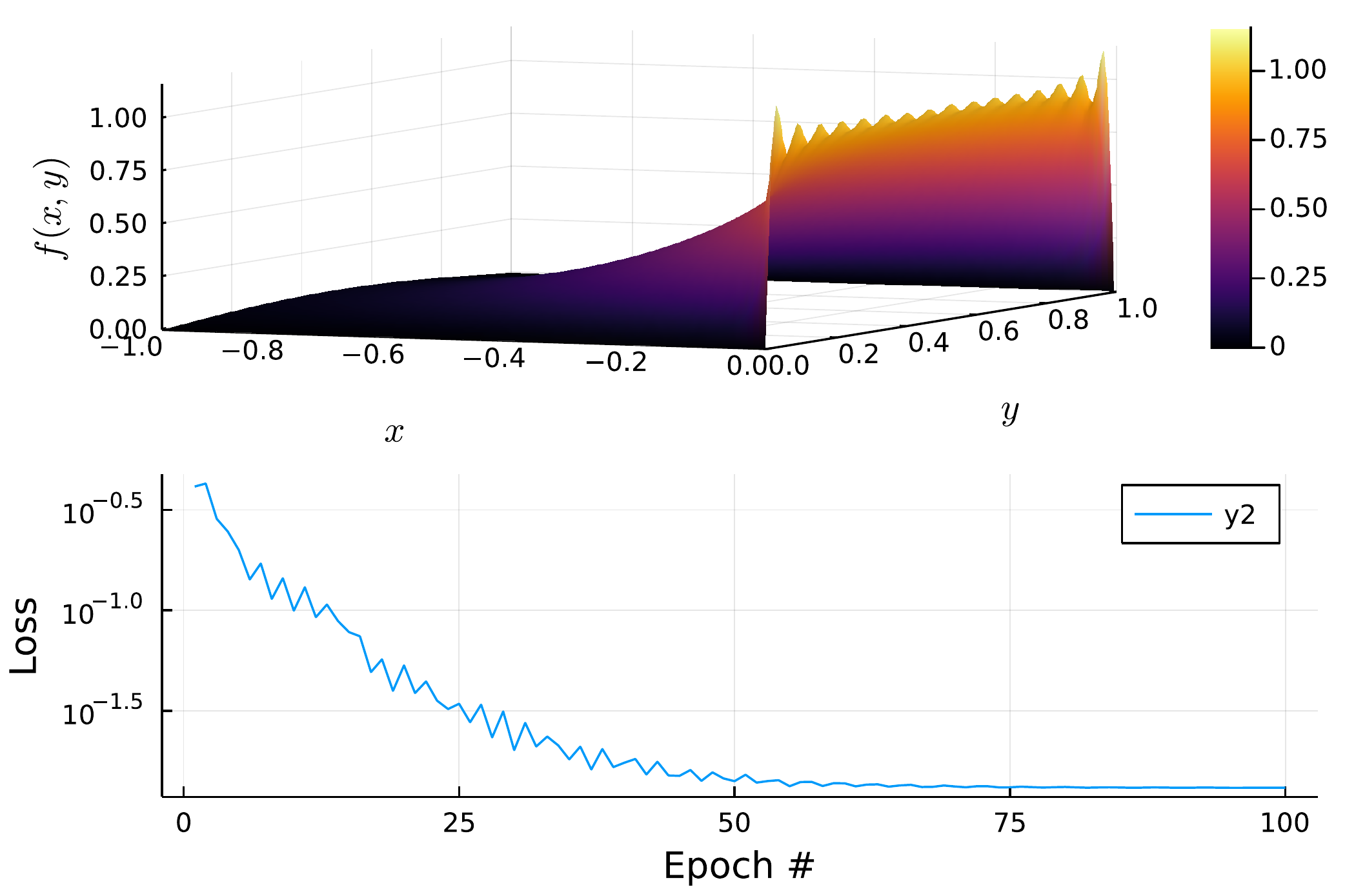}
    \caption{(Top) Converged QNN output for solving Laplace's equation on a unit square domain with boundary conditions setting three sides to $f(x,y)=0$ and one side to $f(x,y)=1$. (Bottom) Loss profile as a function of epoch \# in the QNN training process.}
    \label{fig:SI_plot_qnn}
\end{figure}

\subsection{Quantum Imaginary- and Complex-Time Evolution}
\label{app:qite}
So far we have assumed we can directly substitute a complex number into the $t$ variable representing time evolution over a Hamiltonian. But how would one implement complex-time evolution in a real quantum system? In this section, we will highlight strategies known in the literature as Quantum Imaginary Time Evolution (QITE) and extend these to general Quantum Complex Time Evolution (QCTE).

QITE is a strategy that was first used in classical computing, in particular in the field of computational quantum chemistry, to prepare ground-states/energies for certain types of Hamiltonians, including molecular and lattice Hamiltonians. More recently, it was shown that quantum algorithms could offer a substantial speedup over classical methods. There are many variations of QITE, but two main approaches are the Trotterized-version by Motta et al.  \cite{Motta2020} and the variational approach described by McArdle et al. \cite{McArdle2019}.

In Motta et al. \cite{Motta2020}, the overall imaginary time evolution step is expanded into a series of Trotter steps. Each step evolves over only a very short time. Based on locality properties, one can construct an approximate linear system that is solved on a classical computer to find a proxy unitary operator that can replace the non-unitary step, including the appropriate normalization factor. In our work, that normalization factor would be re-multiplied as a scalar number in classical post-processing to capture a desirable decaying effect.

Because the proposal is already based on a Trotterization-approach, it is possible to show that the QITE can be extended to QCTE by interleaving imaginary-time-evolution Trotter steps with real-time-evolution Trotter steps as
\begin{equation}
    e^{-\beta (\hat{H}_r + i\hat{H}_i)} = (e^{-\Delta \tau \hat{h}_r[1]} e^{-i\Delta \tau \hat{h}_i[1]} e^{-\Delta \tau \hat{h}_r[2]} e^{-i\Delta \tau \hat{h}_i[2]}\ldots)^n + \mathcal{O}(\Delta \tau)
\end{equation}
where we decomposed both $\hat{H}_r$ and $\hat{H}_i$ into sums of mutually commuting operators $\hat{h}_r[m]$ and $\hat{h}_i[m]$ respectively, and $n=\frac{\beta}{\Delta \tau}$.

McArdle et al. \cite{McArdle2019} aims to circumvent the impractical circuit depth requirements of \cite{Motta2020} by taking an approximate, variational approach, now also referred to as VarQITE (Variational QITE). McLachlan's variational principle is used to map the imaginary-time evolved states onto a suitable, chosen variational ansatz, allowing to effectively perform QITE on a quantum device. 
Also in this case, QCTE proposes to replace the real Hamiltonian term by a more general complex Hamiltonian. 
However, future work would still be required to adapt the VarQITE estimate routine of the normalization factor to our case.

\section{Divergence-Free (Quantum) Neural Networks}
\label{app:divfree}

We use exterior derivatives on differential forms to construct divergence-free networks in arbitrary dimensions. To construct divergence-free networks in an $N$-dimensional space, we can represent a general $(N-2)$-form with a neural network. Taking the exterior derivative of this network yields a divergence-free $(N-1)$-form. This follows since
\begin{enumerate}
    \item There is a correspondence between the exterior derivatives of $(N-1)$-forms and divergence operators.
    \item Nested exterior derivatives evaluate to zero.
\end{enumerate}
A full treatment of differential forms lies outside the scope of this work (see e.g. \cite{perot2014, weintraub2014} for more detailed coverage), so we focus instead on practical demonstrations aimed to facilitate the derivation of divergence-free neural networks in further work.

\subsection{Two Dimensions}
\label{app:divfree_2d}
While divergence-free fields in two dimensions might be constructed by inspection, we use the two-dimensional case as a showcase of the methodology involving differential forms.

Start with the definition of a general zero-form:
\begin{equation}\label{eq:2d_zero_form}
f = y\left(x_1, x_2\right),
\end{equation}
on which we take an exterior derivative:
\begin{equation}\label{eq:2d_one_form}
df = \pd{y}{x_2}dx_2 + \pd{y}{x_1}dx_1.
\end{equation}

Consider an exterior derivative on a general $1$-form $g = adx_2 - bdx_1$:
\begin{align}\label{eq:2d_general_one_form}
dg &= \left(\pd{a}{x_1}dx_1 + \pd{a}{x_2}dx_2\right) dx_2 - \left(\pd{b}{x_1}dx_1 + \pd{b}{x_2}dx_2\right) dx_1 \\
  &=  \left(\pd{a}{x_1} + \pd{b}{x_2}\right)dx_1dx_2,
\end{align}
where we use the identity on differential forms that $dx_idx_j = -dx_jdx_i$, which also implies that $dx_idx_i=0$.

This demonstrates the correspondence between the exterior derivative of an ($N-1$)-form and the divergence operator.

Since $d\left(d\left(f\right)\right) = 0$ for any differential form $f$, equating coefficients between $df$ and $g$ then yields a formula for a divergence-free field in two dimensions. So, given a map $(x_1, x_2) \rightarrow y$, then the field:
\begin{equation}\label{eq:2d_divergence_free}
\left(\pd{y}{x_2}, -\pd{y}{x_1}\right)
\end{equation}
is divergence-free.

\subsection{Four Dimensions}
\label{app:divfree_4d}

Since the three-dimensional case is well-handled by standard vector calculus, we turn to four dimensions as a final demonstration of practical calculations of divergence-free fields with differential forms. We note that the increasing complexity of these calculations might make them better suited for calculation using a computer algebra system.

We start by defining a general $2$-form in four dimensions:
\begin{equation}\label{eq:4d_2form}
f = y_1 dx_1dx_2 + y_2 dx_1 dx_3 + y_3 dx_1dx_4 + y_4 dx_2dx_3 + y_5 dx_2dx_4 + y_6 dx_3dx_4.
\end{equation}

Taking exterior derivatives results in:
\begin{align}\label{eq:4d_3form}
df = &\left(\pd{y_4}{x_4} - \pd{y_5}{x_3} +\pd{y_6}{x_2}\right) dx_2dx_3dx_4 \nonumber\\
   +&\left(-\pd{y_2}{x_4} + \pd{y_3}{x_3} - \pd{y_6}{x_1}\right)dx_3dx_4dx_1 \nonumber\\
   +&\left(\pd{y_1}{x_4} - \pd{y_3}{x_2} + \pd{y_5}{x_1}\right)dx_4dx_1dx_2 \nonumber \\
   +&\left(-\pd{y_1}{x_3} +\pd{y_2}{x_2} -\pd{y_4}{x_1}\right)dx_2dx_3dx_4,
\end{align}
where again we use $dx_idx_j = -dx_jdx_i$ and $dx_idx_i=0$.

Considering the exterior derivative of general $3$-form $g = adx_2dx_3dx_4 + bdx_3dx_4dx_1 + cdx_4dx_1dx_2 +ddx_2dx_3dx_4$ yields:
\begin{equation}
    dg = \left(\pd{a}{x_1} + \pd{b}{x_2} + \pd{c}{x_3} + \pd{d}{x_4}\right)dx_1dx_2dx_3dx_4.
\end{equation}

Now, equating coefficients between $df$ and $g$ yields a divergence-free network in four dimensions, i.e. given six-dimensional map $(x_1, x_2, x_3, x_4, x_5, x_6) \rightarrow (y_1, y_2, y_3, y_4, y_5, y_6)$ the map
\begin{equation}\label{eq:curl4d}
\begin{split}
\bigg( & \pd{y_4}{x_4} - \pd{y_5}{x_3} +\pd{y_6}{x_2},\\ -&\pd{y_2}{x_4} + \pd{y_3}{x_3} - \pd{y_6}{x_1},  \\
& \pd{y_1}{x_4} - \pd{y_3}{x_2} + \pd{y_5}{x_1},\\ -&\pd{y_1}{x_3} +\pd{y_2}{x_2} -\pd{y_4}{x_1} \bigg).
\end{split}
\end{equation}
is divergence-free.

 \section{Experimental Details}
 \label{app:exp_details}
 
 The classical neural networks we use are comparatively lightweight, and training for all conventional NNs was done on a single Apple M1 chip running Python 3.9.12 on macOS 12.3.1. 
Equivalently, experiments involving QNNs were run on an AMD\textsuperscript{\copyright} Ryzen 7 3700x processor, in a Python 3.9.5 Conda environment on Ubuntu 18.04.  
Each experiment was repeated 10 times for each method to ensure the repeatability of the outcomes, with the attained average values and standard deviation from the mean reported in the main paper tables.

\paragraph{Quantum Machine Learning} All the QNNs used in this paper are implemented with proprietary code, leveraging upon the packages PyTorch and Yao.jl~\cite{luo2020yao}. 
Details of the various components introduced in this paragraph can be found in~\cite{kyriienko2021dqc}.
We consider two types of QNNs: 
for the first, of type \eqref{eq:simply_connected_qnn} with $|\Psi_{\bm \theta,x,y}\rangle = \hat{\mathcal{U}}_{\bm \theta_2}\textrm{QFM}(x,y)\hat{\mathcal{U}}_{\bm \theta_1}|\emptyset\rangle$, we choose to decompose the QFM as two parallel \emph{Chebyshev-tower} QFM's that apply a $Y$ rotation gate, each spanning 4 qubits. These encoding gates are sandwiched between two layers of variational blocks: a pair of parallel 4-qubit variational blocks before, and a single 8-qubit block after the QFM. All variational blocks employ $\texttt{CNOT}$ entangling gates interleaved with single-qubit parametrized rotations. The training was performed by a hybrid optimization scheme comprising 1200 epochs of Adam \cite{kingma2014}, followed by 300 epochs of L-BFGS~\cite{liu1989}, both with a learning rate $0.05$. To extract information at the end of the evolution, we employ the total magnetization operator. 
The overall circuit diagram is illustrated for clarity in the Appendix.
The second QNN that we use is the quantum harmonic neural network of type \eqref{eq:harmonic_simply_connected_qnn} with a complex-exponential QFM, N=4 qubits, and VB's $\hat{\mathcal{U}}_{\theta_k}$ with depths 8 for both $k=\{1,2\}$.

\begin{figure}
    \centering
    \includegraphics[width=0.7\linewidth]{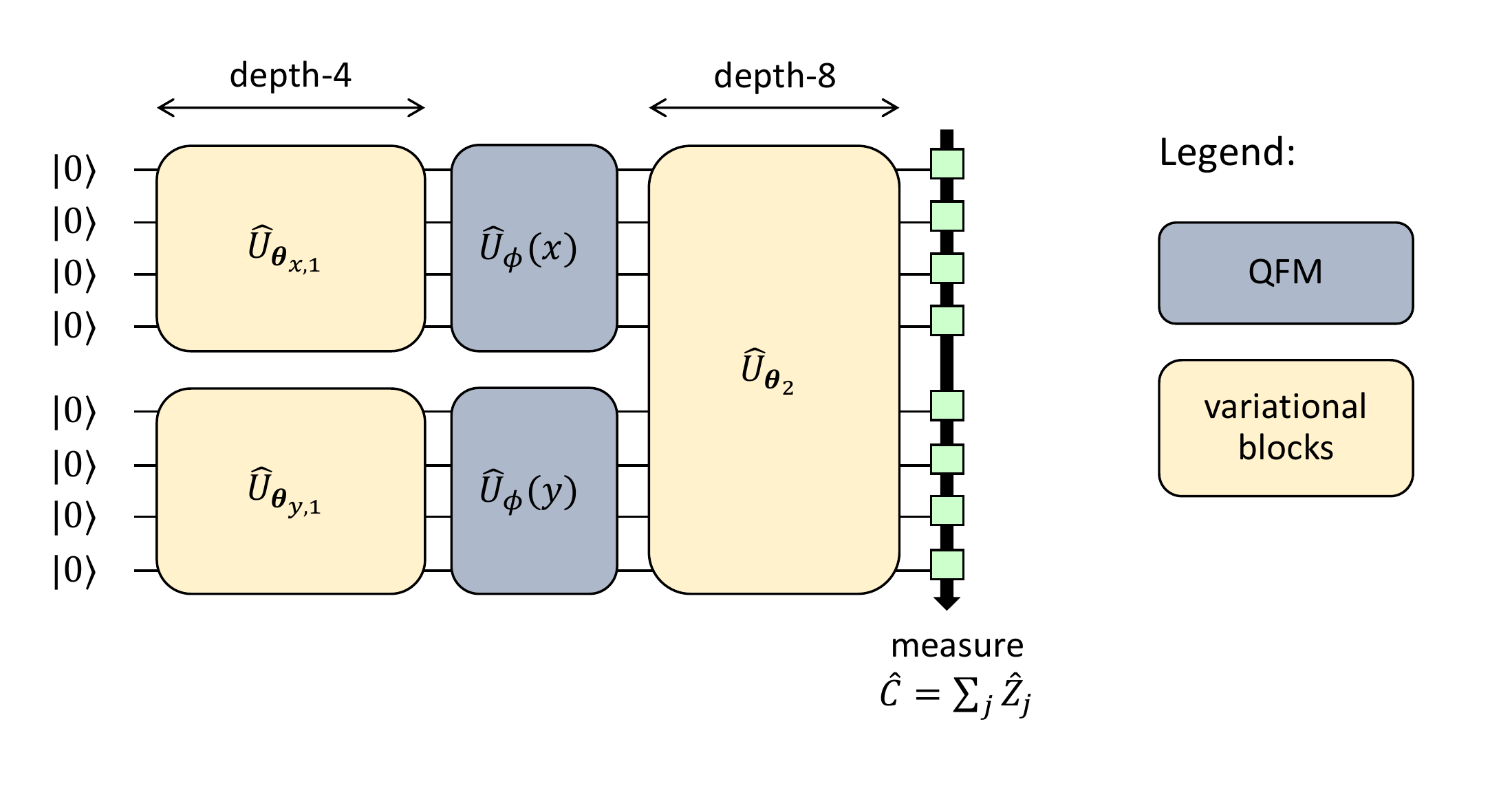}
    \caption{The overall 8-qubit QNN circuit structure, employed as described in Sect. 3.2 of the main paper. Further details on the Quantum Feature Map (QFM) and (hardware-efficient) variational blocks can be found in \cite{kyriienko2021dqc}.
    }
    \label{fig:SI_circuit_qnn}
\end{figure}


\begin{figure*}[h!]
\centering
\includegraphics[width=.98\textwidth]{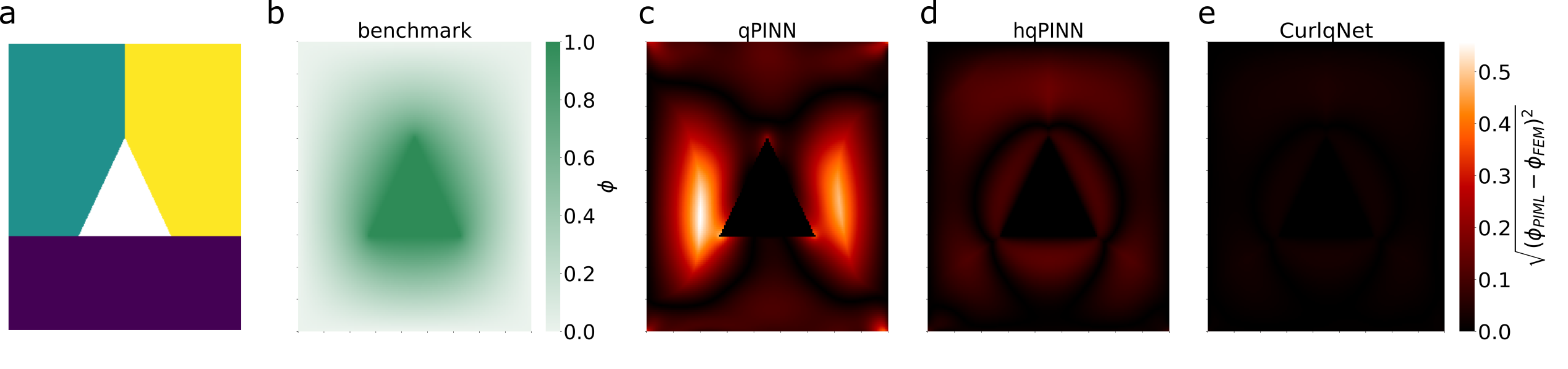}
\caption{
(a) A partitioning of the multiply-connected domain in Sect.\ref{sec:application_heat} into three simply-connected domains on which we define our multiholomorphic and XPINN networks. Note that the white triangle in the centre is not considered a part of the domain.
(b) The solution $\phi(x,y)_{\textrm{FEM}}$ of the problem of a triangular heater positioned in a square box with Dirichlet boundary conditions, as provided by a FEM solver benchmark.
(c-e) A comparison of the performance for different QNN architectures, by plotting the RSE when comparing against $\phi(x,y)_{\textrm{FEM}}$ the solution $\phi(x,y)_{\textrm{PIML}}$ provided respectively by (c) a na{\"i}ve (qPINN), (d) an hqPINN and (e) a CurlqNet architecture.
Additional details and definitions are provided in the main text. 
All plots report a sample of the reference metric on a uniform $150\times 150$ grid. Points within the subdomain belonging to the heater have been excluded from the analysis. 
}
\label{fig:dqc_comparative_performances}
\end{figure*}

\begin{figure*}[h!]
\centering
\includegraphics[width=.9\textwidth]{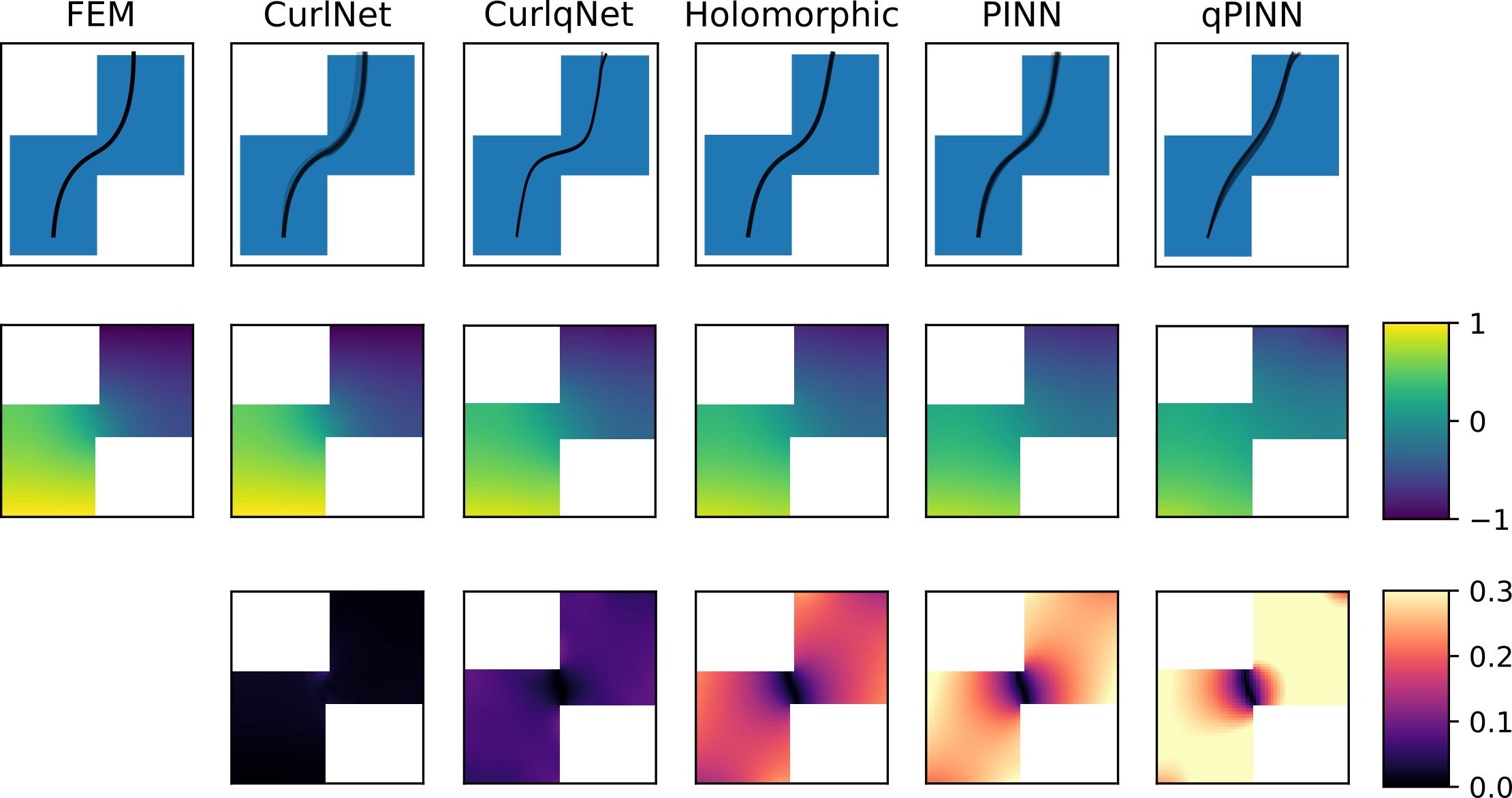}
\caption{
Figures generated from the robot navigation task in Sect.\ref{sec:application_robot} over 10 different network initialisations. First row: robot navigation paths overlapped, with visualization of the domain. Second row: the potential $\phi$ generated to guide the robot, as averaged over the 10 exemplary runs. Third row: the absolute error between the potential $\phi$ generated for each method, and the corresponding finite elements solution. First column: reference finite element solutions, second to sixth columns: results attained with corresponding neural architectures.
}
\label{fig:classical_robot_results_fig}
\end{figure*}

%

\end{document}